\documentclass[10pt, conference, compsocconf]{IEEEtran}
\IEEEoverridecommandlockouts
\usepackage{amsmath,amssymb,amsfonts}
\usepackage{algorithmic}
\usepackage{graphicx}
\usepackage{textcomp}
\usepackage{xcolor}
\def\BibTeX{{\rm B\kern-.05em{\sc i\kern-.025em b}\kern-.08em
    T\kern-.1667em\lower.7ex\hbox{E}\kern-.125emX}}

\usepackage{microtype}

\usepackage[bookmarks=false]{hyperref}
\usepackage[T1]{fontenc}

\usepackage{amsthm}
\usepackage{adjustbox}
\usepackage{etex}
\usepackage{subfigure}
\usepackage{graphicx}
\usepackage{bm}
\usepackage{url}
\usepackage{stmaryrd}
\usepackage{balance}
\usepackage{flushend}
\usepackage{epstopdf}
\usepackage{multirow}
\usepackage{booktabs}
\usepackage{algorithm}
\usepackage{array}
\newcolumntype{C}[1]{>{\centering\arraybackslash}p{#1}}

\renewcommand{\algorithmicrequire}{\textbf{Input:}}
\renewcommand{\algorithmicensure}{\textbf{Output:}}

\newtheorem{theorem}{Theorem}

\newcommand{\ours}[0]{$\text{MMC}$}

\usepackage{hyperref}% For email addresses

\usepackage[square,sort,comma,numbers]{natbib}
\ifCLASSINFOpdf
\else
\fi
\begin{document}

\title{A Self-Organizing Tensor Architecture for Multi-View Clustering}

\author{\IEEEauthorblockN{
      Lifang He\IEEEauthorrefmark{1},
      Chun-Ta Lu\IEEEauthorrefmark{2},
      Yong Chen\IEEEauthorrefmark{3},
      Jiawei Zhang\IEEEauthorrefmark{4},
      Linlin Shen\IEEEauthorrefmark{5},
      Philip S. Yu\IEEEauthorrefmark{2},\IEEEauthorrefmark{6} and
      Fei Wang\IEEEauthorrefmark{1}}
\IEEEauthorblockA{\IEEEauthorrefmark{1}Weill Cornell Medicine; \IEEEauthorrefmark{2}University of Illinois at Chicago; \IEEEauthorrefmark{3}University of Pennsylvania; \IEEEauthorrefmark{4}Florida State University; \\ \IEEEauthorrefmark{5}Shenzhen University; \IEEEauthorrefmark{6}Shanghai Institute for Advanced Communication and Data Science, Shanghai Key Laboratory \\ of Data Science, Fudan University. Email: \{lifanghescut, lucangel\}@gmail.com; ychen123@pennmedicine.upenn.edu; \\
jiawei@ifmlab.org; llshen@szu.edu.cn; psyu@uic.edu; few2001@med.cornell.edu}
}

% \author{\IEEEauthorblockN{
%       Lifang He\IEEEauthorrefmark{1},
%       Chun-Ta Lu\IEEEauthorrefmark{2},
%       Yong Chen\IEEEauthorrefmark{3},
%       Jiawei Zhang\IEEEauthorrefmark{4},
%       Linlin Shen\IEEEauthorrefmark{5},
%       Philip S. Yu\IEEEauthorrefmark{6},\IEEEauthorrefmark{7} and
%       Fei Wang\IEEEauthorrefmark{1}}
% \IEEEauthorblockA{\IEEEauthorrefmark{1}Weill Cornell Medical College, Cornell University, NY, USA; lifanghescut@gmail.com, few2001@med.cornell.edu}
% \IEEEauthorblockA{\IEEEauthorrefmark{2}Google Research, CA, USA; lucangel@gmail.com}
% \IEEEauthorblockA{\IEEEauthorrefmark{3}Department of Biostatistics and Epidemiology, University of Pennsylvania, PA, USA;
% ychen123@pennmedicine.upenn.edu}
% \IEEEauthorblockA{\IEEEauthorrefmark{4}IFM Lab, Department of Computer Science, Florida State University, FL, 
% USA; jiawei@ifmlab.org} 
% \IEEEauthorblockA{\IEEEauthorrefmark{5}Department of Computer Science and Software Engineering, Shenzhen University, Shenzhen, China; llshen@szu.edu.cn}
% \IEEEauthorblockA{\IEEEauthorrefmark{6}Department of Computer Science, University of Illinois at Chicago, Chicago, IL, USA; psyu@uic.edu}
% \IEEEauthorblockA{\IEEEauthorrefmark{7}Shanghai Institute for Advanced Communication and Data Science, Fudan University, Shanghai, China}
% }

\maketitle

\begin{abstract}
In many real-world applications, data are often unlabeled and comprised of different representations/views which often provide information complementary to each other. 
Although several multi-view clustering methods have been proposed, most of them routinely assume one weight for one view of features, and thus inter-view correlations are only considered at the view-level. These approaches, however, fail to explore the explicit correlations between features across multiple views. 
In this paper, we introduce a tensor-based approach to incorporate the higher-order interactions among multiple views as a tensor structure. Specifically, we propose a multi-linear multi-view clustering (MMC) method that can efficiently explore the full-order structural information among all views and reveal the underlying subspace structure embedded within the tensor. Extensive experiments on real-world datasets demonstrate that our proposed MMC algorithm clearly outperforms other related state-of-the-art methods.
\end{abstract}


%\IEEEraisesectionheading{\section{Introduction}\label{sec:intro}}
\section{Introduction}\label{sec:intro}
In many applications, data are naturally comprised of multiple representations or \emph{views}. For example, images on the web have textual descriptions associated with them; one document may be translated into multiple different languages. Observing that different views often provide compatible and complementary information, it appears natural to integrate them together for better performance rather than relying on a single view. The key of learning from multiple views is to leverage the interactions and correlations between views in order to outperform simply concatenating views. As multi-view data without labels are plentiful in real world applications, the problem of unsupervised learning from unlabeled multi-view data has attracted considerable attention in recent years~\cite{chao2017survey,sun2017sequential,shao2016online}, referred to as multi-view clustering.

Most of existing multi-view clustering algorithms are essentially extended from classical single-view clustering algorithms, such as  spectral clustering and K-means clustering. 
For example, a co-regularized multi-view spectral clustering~\cite{coregSC} is developed by keeping the consistency to the same clustering across all of similarity graphs. Some works~\cite{shao2016multi,AMGL} adaptively learn weights to differ the reliability of different views for each graph during the optimization. 
Although these multi-view spectral clustering algorithms can achieve promising performance, they still have two main drawbacks: First, these algorithms generally need to build a proper similarity graph for each view. The construction of the similarity graph is a key issue involving many factors, such as the choice of kernels and their parameters. These factors may greatly affect the final clustering performance. Second, due to the heavy computation of similarity graphs, these graph based methods cannot effectively tackle large-scale multi-view clustering problems. 

In contrast, multi-view K-means clustering approaches are more practically useful to deal with large-scale data since they are free from constructing similarity graphs. This kind of methods is originally derived from the non-negative matrix factorization (NMF) with orthogonality constraints, which is equivalent to relaxed K-means clustering~\cite{ding2005equivalence}. \cite{multiNMF} proposed a joint NMF based multi-view clustering algorithm via seeking for a factorization that gives compatible clustering solutions across multiple views. 
~\cite{RMKMC} proposed the robust multi-view K-means clustering by using the structured sparsity-inducing norm, $\ell_{2,1}$-norm, to replace the Frobenius-norm in the K-means clustering objective and learning individual weight for each view. After \cite{SEC} showed that a regression-like clustering objective is equivalent to the Discriminative K-means, ~\cite{wang2013multi} incorporated the group $\ell_1$-norm together with the $\ell_{2,1}$-norm as joint structured regularization terms in the objective to better capture the view-wise relationships among data. 
However, these algorithms are performed in the original feature space without any discriminative subspace learning mechanism that may render curse of dimensionality when dealing with multi-view and high dimensional data. 

Furthermore, most previous multi-view clustering approaches assume one weight for one view of features, and thus inter-view correlations are only considered at the view-level. These approaches, however, fail to explore the explicit correlations between features across multiple views. Recently, several researchers proposed the use of tensor analysis method to address multi-view clustering problems \cite{SCMV_3DT,cao2017t,sun2018mega}, and it is reported to achieve competitive performance compared with conventional multi-view clustering algorithms. However, existing methods mainly focus on the third-order tensor representation, while the higher-order structural information among all views has not been fully exploited. In this paper, we propose a higher-order tensor based Multi-linear Multi-view Clustering (MMC) method for exploring the unsupervised heterogeneous data fusion and clustering, which can take all the possible interactions between views and clusters into account, ranging from the first-order interactions to the highest order interactions.

% The contributions of this paper are summarized as follows: 
% \begin{itemize}
%   \item We present an innovative regression-like method for multi-view clustering by effectively organizing multi-view data into higher-order tensor representation, which can take all the possible interactions between multiple views into account, ranging from the first-order interactions (\emph{i.e.}, contributions of single features) to the highest order interactions (\emph{i.e.}, contributions of combinations of features from all views).
%   \item To the best of our knowledge, it is the first work to propose a multi-view clustering method in the higher-order tensor space. By using tensor factorization, we actually learn the latent subspaces embedded within the tensor in an elegant and compact fashion, without the need to construct the tensor data physically.
% %   \item We propose a coupled sparse matrix factorization (CSMF) model together with an efficient optimization algorithm.
%   \item We present an efficient algorithm to solve the MMC optimization problem, which is proved to converge.
% \end{itemize}

% The remainder of this paper is organized as follows. Section \ref{sec:related} discusses related work. In Section \ref{sec:prelim}, we introduce some basic concepts of tensor and describe the problem definition. The proposed {\ours} approach is presented in Section \ref{sec:method}. Section \ref{sec:exp} describes the experiments and the analysis of the results. Finally, Section \ref{sec:conclude} concludes the paper.

\section{Preliminary}\label{sec:prelim}
% The key to this work is to apply the tensor structure to fuse all possible dependence relationships among different views and different clusters. We begin by introducing some preliminary material on tensor algebra (a.k.a. multi-linear algebra), and then state the problem of multi-view clustering. 

% Table \ref{tab:notation} lists basic symbols that will be used throughout the paper.

\subsection{Tensor Algebra and Notation}\label{sec:algebra}
We denote scalars by lowercase letters, e.g., $x$; vectors by boldfaced lowercase letters, e.g., $\mathbf{x}$; matrices by boldface uppercase letters, e.g., $\mathbf{X}$; and tensors by calligraphic letters, e.g., $\mathcal{X}$. We denote their entries by $x_i$, $x_{i,j}$, $x_{i,j,k}$, etc., depending on the number of modes. All vectors are column vectors unless otherwise specified. For an arbitrary matrix $\mathbf{X} \in \mathbb{R}^{I \times J}$, its $i$-th row and $j$-th column vector are denoted by $\mathbf{x}^{i}$ and $\mathbf{x}_{j}$, respectively. Its Frobenius norm is defined by $\left \| \mathbf{X} \right \|_F = \sqrt{\sum_{i=1}^{I} \left \|\mathbf{x}^{i} \right \|_2^2}$, and its $\ell_{2,1}$ norm is defined by $\left \| \mathbf{X} \right \|_{2,1} = \sum_{i=1}^{I} \left \| \mathbf{x}^{i} \right\|_2$. 
Additionally, $tr(\cdot)$ denotes the trace function, $vec(\cdot)$ denotes the column stacking operator, $diag(\cdot)$ denotes a diagonal matrix formed from its vector argument, $\ast$ denotes the Hadamard (elementwise) product, and $\otimes$ denotes the Kronecker product. 

An important property of Hadamard product is $\mathbf{a} \ast \mathbf{b} = diag(\mathbf{a}) \mathbf{b}$. 
An important application of Kronecker product is to rewrite the matrix equation $\mathbf{A} \mathbf{X} \mathbf{B} = \mathbf{C} $ into the equivalent vector equation $(\mathbf{B}^\mathrm{T} \otimes \mathbf{A}) vec(\mathbf{X}) = vec(\mathbf{C}) $. 
The inner product of two tensors $\mathcal{X}, \mathcal{Y} \in \mathbb{R}^{I_1\times\cdots\times I_M}$ is defined by $\big \langle \mathcal{X},\mathcal{Y} \big \rangle = \sum\limits_{i_1=1}^{I_1} \cdots \sum\limits_{i_M=1}^{I_M} x_{i_1, \cdots, i_M} y_{i_1, \cdots, i_M}$.
The outer product of vectors $\mathbf{x}^{(m)} \in \mathbb{R}^{I_m}$ for $m \in [1:M]$ is an $M$-th order tensor and defined elementwise by $\big (\mathbf{x}^{(1)} \circ \cdots \circ \mathbf{x}^{(M)}\big)_{i_1, \cdots, i_M} = x^{(1)}_{i_1}  \cdots x^{(M)}_{i_M} = \prod_{m=1}^M x_{i_m}^{(m)}$ for all values of the indices. In particular, for $\mathcal{X}=\mathbf{x}^{(1)} \circ \cdots \circ \mathbf{x}^{(M)}$ and $\mathcal{Y}=\mathbf{y}^{(1)} \circ \cdots \circ \mathbf{y}^{(M)}$, it holds that
\begin{align}
\big \langle \mathcal{X},\mathcal{Y} \big \rangle = \prod_{m=1}^{M} \big \langle \mathbf{x}^{(m)}, \mathbf{y}^{(m)} \big \rangle = \prod_{m=1}^{M} {\mathbf{x}^{(m)}}^\mathrm{T} \mathbf{y}^{(m)}.
\label{eq:inner_rank1}
\end{align}
% Given a tensor $\mathcal{X} \in \mathbb{R}^{I_1\times\cdots\times I_M}$, its CP decomposition is
% % $$\mathcal{X} = \sum_{r=1}^{R} \mathbf{x}_{r}^{(1)} \circ \cdots \circ \mathbf{x}_{r}^{(M)} = \llbracket \mathbf{X}^{(1)}, \dots, \mathbf{X}^{(M)} \rrbracket,$$
% \begin{equation}
% \mathcal{X} = \sum_{r=1}^{R} \mathbf{x}_{r}^{(1)} \circ \cdots \circ \mathbf{x}_{r}^{(M)} = \llbracket \mathbf{X}^{(1)}, \cdots, \mathbf{X}^{(M)} \rrbracket,
% \end{equation}
% where for $m \in [1:M]$, $\mathbf{X}^{(m)}= [\mathbf{x}_{1}^{(m)}, \cdots, \mathbf{x}_{R}^{(m)} ]$ are factor matrices of size $I_m \times R$, $R$ is the number of factors, and $\llbracket \cdot \rrbracket$ is used for shorthand.

\subsection{Problem Definition}
In the setting of multi-view clustering, assume we are given a dataset with $N$ instances in $V$
views $\{ \mathbf{X}^{(v)}\}_{v=1}^V$, where $\mathbf{X}^{(v)} = \big[\mathbf{x}^{(v)}_1, \dots, \mathbf{x}^{(v)}_N \big] \in \mathbb{R}^{D_v \times N}$ denotes the feature matrix of the $v$-th view and $D_v$ is the dimensionality of the $v$-th view. %The feature vector $\mathbf{x}_n \in \mathbb{R}^{I}$ of the $n$-th instance includes all the features in $V$ views, such that $I=\sum_{v=1}^{V} I_v$. 
Our goal is to partition $N$ instances into $K$ clusters by exploiting the information in all $V$ different views of the data. Specifically, we aim to accurately perform clustering in the joint space of the multiple views, while focusing explicitly on inter-view interactions by virtue of tensor manipulation.
% \section{Methodology}\label{sec:method}
\section{Multi-Linear Multi-View Clustering} \label{sec:method}
Clustering is unsupervised and exploratory in nature. Yet, it can be performed through penalized regression with grouping pursuit. The linear regression-based clustering models often have the following general form of objective function \cite{wang2013multi}:
\begin{equation}
\min_{\mathbf{W}, \mathbf{F}^{\mathrm{T}}\mathbf{F} = \mathbf{I}} \| \mathbf{X}^\mathrm{T}\mathbf{W} + \mathbf{1}_N \mathbf{b}^\mathrm{T} - \mathbf{F}\|^2_F, 
\label{eq:regCluster_b}
\end{equation}
where $\mathbf{X} = \big[\mathbf{x}_1, \dots, \mathbf{x}_N\big] \in \mathbb{R}^{D \times N}$ is a data matrix with $D$ features and $N$ samples, $\mathbf{W} \in \mathbb{R}^{D \times K}$ is the weight matrix (a.k.a. projection matrix), $\mathbf{b} \in \mathbb{R}^{K}$ is the bias vector, $\mathbf{1}_N \in \mathbb{R}^N$ is the constant vector of all ones, $\mathbf{F}= \big[\mathbf{f}_1, \dots, \mathbf{f}_K\big] \in \mathbb{R}^{N\times K}$ is the cluster indicator matrix, and $f_{n,k}$ indicating how likely the $n$-th instance belongs to the $k$-th cluster.
% $\mathbf{W}$ is a weight matrix that accounts for the correlation between response variables

Previous work \cite{SEC} showed the regression-based clustering objective of Eq.~(\ref{eq:regCluster_b}), which is equivalent to the Discriminative $K$-means, obtains better results than $K$-means or spectral clustering methods. Thus, it is desirable to extend regression method to conduct multi-view clustering. However, due to the self-limitation of linear regression, it is not effective to directly use it for multi-view clustering, which only captures linear correlations between views, and neglects all interactions between multiple views. Here we sought to design a regression-based Multi-linear Multi-view Clustering (MMC) method to fuse all possible dependence relationships among different views and clusters, thus result in more accurate and interpretable model.

\subsection{Model}
% \textbf{{\ours} model.} 
Let $\mathbf{z}_n = [\mathbf{x}_n; 1] \in \mathbb{R}^{D+1}$ for $n \in [1:N]$ and $\mathbf{W} = [\mathbf{W}; \mathbf{b}^\mathrm{T}] \in \mathbb{R}^{(D+1) \times K}$, then the bias $\mathbf{b}$ of Eq.~(\ref{eq:regCluster_b}) can be absorbed to $\mathbf{W}$. Eq.~(\ref{eq:regCluster_b}) can thus be rewritten as follows:
\begin{equation}
\| \mathbf{Z}^\mathrm{T}\mathbf{W} - \mathbf{F}\|^2_F = \sum_{n=1}^{N}\sum_{k=1}^{K} (\mathbf{z}_n^\mathrm{T}\mathbf{w}_k - f_{n,k})^2,
\label{eq:regCluster}
\end{equation}
where $\mathbf{Z} = \big[\mathbf{z}_1, \dots, \mathbf{z}_N\big] \in \mathbb{R}^{(D+1) \times N}$, $\mathbf{w}_k$ is the $k$-th column vector of $\mathbf{W}$, and $f_{n,k}$ is the $n$-th row and the $k$-th column of the matrix $\mathbf{F}$. From the regression point of view, we note that $\hat{f}_{n,k}$ can be explicitly expressed as
\begin{align}
\hat{f}_{n,k} = \mathbf{z}_n^\mathrm{T}\mathbf{w}_k = 
\langle \mathbf{e}_k^\mathrm{T}, \mathbf{z}_n^{\mathrm{T}} \mathbf{W} \rangle
= \left\langle \mathbf{z}_n \circ \mathbf{e}_k, \mathbf{W} \right\rangle,
\label{eq:reg_fun}
\end{align}
where $\mathbf{e}_k \in \mathbf{R}^{K}$ is the cluster/class indicator vector and defined as follows:
\[
\mathbf{e}_k = \underbrace{[0, \cdots, 0}_\text{k-1}, 1, 0, \cdots, 0 ]^\mathrm{T} \in \mathbb{R}^K.
\]

On the other hand, by means of the outer product, each multi-view data point $\{\mathbf{x}_n^{(v)}\}_{v=1}^V$ are able to be incorporated into a tensor representation in the form $\mathbf{x}_n^{(1)} \circ\cdots \circ \mathbf{x}_n^{(V)}$ \cite{cao2014tensor}. Crucially, by adding an extra constant feature to each view data $\mathbf{x}_n^{(v)}$, i.e., $\mathbf{z}_n^{(v)} = [\mathbf{x}_n^{(v)}; 1] \in \mathbb{R}^{D_v + 1}$ for $v \in [1:V]$, this paradigm can take advantage of different orders of interactions between views \cite{lu2017multilinear}. Thus, following this recipe, we define the tensor representation to exploit all interactions between views by
%Here we form a multilinear model by structuring the weight parameters of all clusters into a tensor.  
%Let $\mathbf{e}_k \in \mathbb{R}^{T}$ denotes the cluster indicator vector for the $t$-th cluster, where $e_{t,t} = 1$ and $e_{t,s} = 0$ for $s \neq t$. 
\[
\mathcal{Z}_n = \mathbf{z}^{(1)}_n \circ\cdots \circ \mathbf{z}^{(V)}_n \in \mathbb{R}^{(D_1+1) \times \cdots \times (D_V+1)},
\]
where elementwise $(\mathcal{Z}_n)_{d_1, \cdots, d_V} = \prod_{v=1}^V z_{n d_v}^{(v)}$, and $z_{n d_v}^{(v)}$ is the $d_v$-th element of $\mathbf{z}^{(v)}_n$.
% where elementwise $(\mathcal{Z}_n)_{i_1 \cdots i_V} = z^{(1)}_{n,i_1} \cdots z^{(V)}_{n,i_V} = \prod_{v=1}^V z_{n,i_v}^{(v)}$.

Using the above input, we can write Eq.~(\ref{eq:reg_fun}) as
\begin{align}
\hat{f}_{n,k} = \left\langle \mathcal{Z}_n \circ \mathbf{e}_k,\mathcal{W} \right\rangle,
\label{eq:nest}
\end{align}
where $\mathcal{W}=\{w_{d_1,...,d_{V},k}\} \in \mathbb{R}^{(D_1+1) \times\cdots\times (D_{V}+1) \times K}$ is the weight tensor to be learned. Notice that $w_{d_1,\dots,d_{V},k}$ with some indexes satisfying $d_v=D_v + 1$ encodes lower-order interactions between other views except the $v$-th view.

Without any additional assumptions on $\mathcal{W}$, the above model is apparently severely over-parameterized. Besides, the inter-view and inter-cluster relationships in different orders are not jointly explored, as the weight parameters are modeled from the independent sum of each order. The merit of the above formulation lies in some suitable low-dimensional structures imposed on $\mathcal{W}$. We assume that $\mathcal{W}$ has a low rank and can be factorized by CP decomposition \cite{kolda2009tensor} as
\begin{equation}
    \mathcal{W} = \sum_{r=1}^{R} \mathbf{w}_{r}^{(1)} \circ \cdots \circ \mathbf{w}_{r}^{(V+1)} = \llbracket \mathbf{W}^{(1)}, \cdots, \mathbf{W}^{(V+1)} \rrbracket,
\end{equation}
where $\mathbf{W}^{(v)}= [\mathbf{w}_{1}^{(v)}, \cdots, \mathbf{w}_{R}^{(v)}] \in \mathbb{R}^{(D_v+1) \times R}$ for $v \in [1:V]$, and $\mathbf{W}^{(V+1)}= [\mathbf{w}_{1}^{(V+1)}, \cdots, \mathbf{w}_{R}^{(V+1)}] \in \mathbb{R}^{K \times R}$.

% \begin{figure*}[t]
%     \centering
%     \includegraphics[width=0.72\linewidth]{MMVC_new.pdf}
%     \caption{An illustration of our self-organizing tensor architecture for multi-view clustering. The feature interactions in multi-view data are modeled in a full-order tensor. The clustering result is learned from the view-specific feature space and the cluster-view shared common feature space.}
%     \label{fig:MMC}
% \end{figure*}

Then Eq.~(\ref{eq:nest}) can be transformed into
\begin{align}
\hat{f}_{n,k} = \sum_{r=1}^{R} \big\langle \mathcal{Z}_n \circ \mathbf{e}_k, \mathbf{w}_{r}^{(1)} \circ \cdots \circ \mathbf{w}_{r}^{(V+1)} \big\rangle.
\label{eq:nest_W_decomp}
\end{align}

% \begin{equation}
% \mathcal{W} = \sum_{r=1}^{R} \mathbf{w}_{r}^{(0)} \circ \cdots \circ \mathbf{w}_{r}^{(V)},
% \label{eq:W_decomp}
% \end{equation}
% Substitute Eq.~(\ref{eq:W_decomp}) into Eq.~(\ref{eq:nest}), we have:
% \begin{align}
% f_k(\{\mathbf{x}_n\}) = \sum_{r=1}^{R} \big\langle \mathbf{e}_k \circ \mathcal{Z}_n, \mathbf{w}_{r}^{(0)} \circ \cdots \circ \mathbf{w}_{r}^{(V)} \big\rangle.
% \label{eq:nest_W_decomp}
% \end{align}

% %\vspace{-5pt}
% %\scriptsize
% \begin{align}
% f_k(\{\mathbf{x}_n\}) &= \sum_{s=1}^{K} \sum_{i_1=0}^{I_1}\cdots \sum_{i_{V}=0}^{I_{V}}  \left(\sum_{r=1}^R \phi_{s,r} \prod_{v=1}^{V} w_{i_v,r}^{(v)}\right) \left(e_{k, s} \prod_{v=1}^{V} z_{n, i_v}^{(v)}\right) \nonumber \\
% &= \sum_{r=1}^R \left( \sum_{s=1}^{K} w^{(0)}_{s,r}e_{k, s} \right)\sum_{i_1=0}^{I_1}\cdots \sum_{i_{V}=0}^{I_{V}}  \left( \prod_{v=1}^{V} w_{i_v,r}^{(v)} z_{n, i_v}^{(v)}\right) \nonumber \\
% %&= \sum_{r=1}^R \left( \sum_{s=1}^{K} w^{(0)}_{s,r}e_{t, s} \right) \left( \sum_{i_1=0}^{I_1} w_{i_1,r}^{(1)} z_{t, i_1}^{(1)}\right) \cdots \left( \sum_{i_{V}=0}^{I_{V}} w_{i_V,r}^{(V)} z_{t, i_V}^{(V)}\right) \nonumber \\
% &= \sum_{r=1}^{R}  \left\langle \mathbf{w}_{r}^{(0)} \circ \mathbf{w}_{r}^{(1)} \circ \cdots \circ \mathbf{w}_{r}^{(V)} ~,~ \mathbf{e}_k  \circ \mathbf{z}_n^{(1)} \circ \cdots \mathbf{z}_n^{(V)} \right\rangle ~.
% \label{eq:mfm_cp}
% \end{align}
%Note that only the $k$-th element of $\mathbf{e}_k$ equals $1$ and all the others are $0$. 
According to Eq.~(\ref{eq:inner_rank1}), we can further rewrite Eq.~(\ref{eq:nest_W_decomp}) as
%\vspace{-5pt}
%\scriptsize
\begin{align}
& \hat{f}_{n,k} =\sum_{r=1}^{R}  w_{k,r}^{(V+1)} \big( \mathbf{z}_n^{(1)^\mathrm{T}} \mathbf{w}_{r}^{(1)} \big) \cdots \big( \mathbf{z}_n^{(V)^\mathrm{T}} \mathbf{w}_{r}^{(V)} \big) \nonumber \\
&= \big( \big( \mathbf{z}_n^{(1)^\mathrm{T}} \mathbf{W}^{(1)} \big) \ast \cdots \ast \big( \mathbf{z}_n^{(V)^\mathrm{T}} \mathbf{W}^{(V)} \big)\big) {\mathbf{w}^{(V+1)k}}^\mathrm{T},
\label{eq:mmvc_no_regular}
\end{align}
where $\ast$ is the Hadamard (elementwise) product, and $\mathbf{W}^{(v)}$ for $v \in [1:V]$ can be considered as the projection matrix that projects the features from $v$-th view into the common subspace. It should be noted that the last row of $\mathbf{W}^{(v)}$ is always associated with $z_{n(D_v+1)}^{(v)}=1$ and represents the bias factors of the $v$-th view. Through the bias factors, the lower-order interactions are explored in the study.

By plugging Eq.~(\ref{eq:mmvc_no_regular}) into Eq.~(\ref{eq:regCluster}), we have
\begin{align}
\sum_{n=1}^{N}\sum_{k=1}^{K} (\hat{f}_{n,k} - f_{n,k})^2 = \big\| \big( \prod_{v=1}^{V} \ast \big( \mathbf{Z}^{(v)^\mathrm{T}} \mathbf{W}^{(v)} \big)\big) {\mathbf{W}^{(V+1)}}^{\mathrm{T}}  - \mathbf{F}\big\|^2_F, \nonumber
% \label{eq:mmvc_obj_main}
\end{align}
where $\mathbf{Z}^{(v)} = \big[\mathbf{z}_1^{(v)}, \dots, \mathbf{z}_N^{(v)}\big] \in \mathbb{R}^{(D_v+1) \times N}$ is the feature matrix of view $v$, and $\mathbf{W}^{(V+1)}$ corresponds to the weight matrix of all clusters. It is clear that the result becomes a genuine joint and self-organizing model, as it enables full-order interaction information borrowing and sharing among the views and clusters.

Moreover, considering that multi-view data are known to contain redundant features, we add an $\ell_{2,1}$-norm regularization on each weight matrix $\mathbf{W}^{(v)}$. The $\ell_{2,1}$ norm promotes row-sparsity, such property makes it suitable for the task of feature selection~\cite{nie2010efficient}. Finally, we design our objective function of {\ours} as follows: 
\begin{align}
\min_{\{\mathbf{W}^{(v)}\}, \mathbf{F}^{\mathrm{T}} \mathbf{F} = \mathbf{I}}~
& \big\| \big( \prod_{v=1}^{V} \ast \big( \mathbf{Z}^{(v)^\mathrm{T}} \mathbf{W}^{(v)} \big)\big) {\mathbf{W}^{(V+1)}}^{\mathrm{T}}  - \mathbf{F}\big\|^2_F \nonumber \\
& + \gamma \sum_{v=1}^{V+1} \| \mathbf{W}^{(v)}\|_{2,1}.
\label{eq:mmvc_obj}
\end{align}
% \begin{equation}
% \begin{adjustbox}{max width=1\columnwidth}
% $
% \min\limits_{\{\mathbf{W}^{(v)}\}, \mathbf{F}^{\mathrm{T}} \mathbf{F} = \mathbf{I}}~
%  \big\| \big( \prod_{v=1}^{V} \ast \big( \mathbf{Z}^{(v)^\mathrm{T}} \mathbf{W}^{(v)} \big)\big) {\mathbf{W}^{(0)}}^{\mathrm{T}}  - \mathbf{F}\big\|^2_F
% + \gamma \sum_{v=0}^V \| \mathbf{W}^{(v)}\|_{2,1}.
% \label{eq:mmvc_obj}
% $
% \end{adjustbox}
% \end{equation}
Where $\gamma$ is the regularized parameter. For convenience, below we let $\mathbf{\Pi} = \prod_{v=1}^{V} \ast (\mathbf{Z}^{(v)^\mathrm{T}} \mathbf{W}^{(v)}) \in \mathbb{R}^{N \times R}$ denote the embedding matrix from all the views, and $\mathbf{\Pi}^{(-v)} = \prod_{v'\neq v}^{V} \ast (  \mathbf{Z}^{(v')^\mathrm{T}} \mathbf{W}^{(v')} ) \in \mathbb{R}^{N \times R}$ denote the embedding matrix from all the other views except the $v$-th view.
% where $\gamma$ is the regularization parameter of $\ell_{2,1}$-norm that induces row-sparsity of the projection/weight matrix $\mathbf{W}^{(v)}$, which leads to the selection of relevant factors, and thus elimination of redundant and even conflicting features. For convenience, in the following we let $\mathbf{\Pi} = \prod_{v=1}^{V} \ast (\mathbf{Z}^{(v)^\mathrm{T}} \mathbf{W}^{(v)})$ denote the embedding matrix from all the views, and $\mathbf{\Pi}^{(-v)} = \prod_{v'\neq v}^{V} \ast (  \mathbf{Z}^{(v')^\mathrm{T}} \mathbf{W}^{(v')} )$ denote the embedding matrix from all the other views except the $v$-th view.

It is easy to see that the parameters of the full-order interactions between multiple clusters with multiple views are jointly factorized in Eq.~(\ref{eq:mmvc_obj}). This forms the basis of our MMC model to fuse all possible dependence relationships among different views and different clusters. Another appealing property of {\ours} comes from the main characteristics of multi-linear analysis. After factorizing the parameter tensor $\mathcal{W}$, this appealing latent model setup enables dimension/parameter reduction and induces dependency among the views and different orders, with no need to explicitly introduce the cluster indicator $\mathbf{e}_k$ and construct a new tensor. In particular, the model complexity is reduced from $O(K\prod_{v=1}^V (D_v+1))$ to $O(R(V + K + \sum_{v=1}^V (D_v+1))$, which is linear in the number of original features.

\subsection{Estimation}
% The global solution to the optimization problem in Eq.(\ref{eq:mmvc_obj}) is difficult to achieve since it is not jointly convex w.r.t. the set of variables $(\{\mathbf{W}^{(v)}\}, \mathbf{F})$. Therefore we propose an alternative iterative algorithm to solve the problem by converting the original problem into a series of sub-problems, where only one variable is updated at a time.
% The objective function in Eq.~(\ref{eq:mmvc_obj}) involves $\{\mathbf{W}^{(v)}\}$ and $\mathbf{F}$, and it is not easy to optimize the problem w.r.t. all the variables simultaneously. Therefore we propose an iterative algorithm to alternatively minimize Eq.~(\ref{eq:mmvc_obj}) for each variable while fixing the other.
The objective function in Eq.~(\ref{eq:mmvc_obj}) is non-convex, and solving for the global minimum is difficult in general. Therefore we derive an efficient iterative algorithm to reach the local optimum, by alternatively minimizing Eq.~(\ref{eq:mmvc_obj}) for each variable while fixing the other. The overall algorithm is summarized in Algorithm \ref{alg:MMVC}.

\textbf{Update $\mathbf{W}^{(v)}(1\leq v \leq V)$}: By keeping all other variables fixed and minimizing over $\mathbf{W}^{(v)}$, we need to solve the following problem:
\begin{align}
    & \min_{\mathbf{W}^{(v)}}~ \| ( \mathbf{\Pi}^{(-v)} \ast (\mathbf{Z}^{(v)^\mathrm{T}} \mathbf{W}^{(v)})) {\mathbf{W}^{(V+1)}}^{\mathrm{T}}  - \mathbf{F}\|^2_F \nonumber \\
    & + \gamma \| \mathbf{W}^{(v)}\|_{2,1}.
\label{eq:opti_Wv}
\end{align}
Due to the nonsmooth regularization term of $\ell_{2,1}$-norm, it is not easy to solve Eq.~(\ref{eq:opti_Wv}) exactly. Inspired by \cite{nie2010efficient}, we instead solve the following problem iteratively to approximate the solution of Eq.~(\ref{eq:opti_Wv}):
% \overset{\triangle}{=}
% \begin{align}
% \min_{\mathbf{W}^{(v)}} \mathcal{J} = ~ & \| ( \mathbf{\Pi}^{(-v)} \ast (\mathbf{Z}^{(v)^\mathrm{T}} \mathbf{W}^{(v)})) {\mathbf{W}^{(0)}}^{\mathrm{T}} - \mathbf{F}\|^2_F \nonumber \\
% & + \gamma tr ({\mathbf{W}^{(v)}}^\mathrm{T} \mathbf{P}^{(v)} \mathbf{W}^{(v)}),
% \label{eq:reform_Wv}
% \end{align}
\begin{align}
\mathcal{J} = & \min_{\mathbf{W}^{(v)}} \| ( \mathbf{\Pi}^{(-v)} \ast (\mathbf{Z}^{(v)^\mathrm{T}} \mathbf{W}^{(v)})) {\mathbf{W}^{(V+1)}}^{\mathrm{T}} - \mathbf{F}\|^2_F \nonumber \\
& + \gamma tr ({\mathbf{W}^{(v)}}^\mathrm{T} \mathbf{P}^{(v)} \mathbf{W}^{(v)}),
\label{eq:reform_Wv}
\end{align}
where $\mathbf{P}^{(v)}$ is a diagonal matrix with the $i$-th diagonal element as $p^{(v)}_{ii} = \frac{1}{2 \|\mathbf{w}^{(v)i} \|_2}$.

% \footnote{When $\|\mathbf{w}^{(v)i} \|_2=0$, Eq.~(\ref{eq:reform_Wv}) is not differentiable. Following \cite{gorodnitsky1997sparse}, we can introduce a small perturbation to regularize the $i$-th diagonal element of $\mathbf{P}^{(v)}$ as $p^{(v)}_{ii} = \frac{1}{2 \sqrt{\|\mathbf{w}^{(v)i} \|_2^2 + \epsilon}}$.}

% \[
% p^{(v)}_{ii} = \frac{1}{2 \|\mathbf{w}^{(v)i} \|_2}
% \]
% Note that $\mathbf{P}^{(v)}$ depends on $\mathbf{W}^{(v)}$ and thus is also unknown variable.
% \begin{equation}
% p^{(v)}_{ii} = \frac{1}{2 \|\mathbf{w}^{(v)}_i \|_2}.
% \end{equation}

Taking the derivative of Eq.~(\ref{eq:reform_Wv}) with respect to $\mathbf{W}^{(v)}$ and setting it to zero, we have:
\begin{align}
& \mathbf{Z}^{(v)} (\mathbf{\Pi}^{(-v)} \ast (( \mathbf{\Pi}^{(-v)} \ast (\mathbf{Z}^{(v)^\mathrm{T}} \mathbf{W}^{(v)})) {\mathbf{W}^{(0)}}^{\mathrm{T}}\mathbf{W}^{(V+1)})) \nonumber \\
& - \mathbf{Z}^{(v)}( \mathbf{\Pi}^{(-v)} \ast ( \mathbf{F} \mathbf{W}^{(0)} ) )  + \gamma \mathbf{P}^{(v)} \mathbf{W}^{(v)} = \mathbf{0}.
\label{eq:gred_Wv}
\end{align}
% \begin{align}
% & \frac{ \partial \mathcal{J} }{ \partial \mathbf{W}^{(v)} } = 2 \mathbf{Z}^{(v)} (\mathbf{\Pi}^{(-v)} \ast (( \mathbf{\Pi}^{(-v)} \ast (\mathbf{Z}^{(v)^\mathrm{T}} \mathbf{W}^{(v)})) {\mathbf{W}^{(0)}}^{\mathrm{T}}\mathbf{W}^{(0)})) \nonumber \\
% & - 2 \mathbf{Z}^{(v)}( \mathbf{\Pi}^{(-v)} \ast ( \mathbf{F} \mathbf{W}^{(0)} ) )  + 2 \gamma \mathbf{P}^{(v)} \mathbf{W}^{(v)} = 0.
% \label{eq:gred_Wv}
% \end{align}
To solve Eq.~(\ref{eq:gred_Wv}) w.r.t. $\mathbf{W}^{(v)}$, we state the following theorem.

\begin{theorem}\label{theo:LS}
Given any matrices $\mathbf{A} \in \mathbb{R}^{M \times N}$, $\mathbf{B} \in \mathbb{R}^{N \times R}$, $\mathbf{C} \in \mathbb{R}^{R \times R}$, $\mathbf{D} \in \mathbb{R}^{M \times M}$ and $\mathbf{E} \in \mathbb{R}^{M \times R}$, the solution of the following equation
\begin{equation}
\mathbf{A} (\mathbf{B} \ast (( \mathbf{B} \ast (\mathbf{A}^{\mathrm{T}} \mathbf{X})) \mathbf{C})) + \mathbf{D} \mathbf{X} = \mathbf{E}
\label{eq:P1}
\end{equation}
is equivalent to the solution of the linear system in the form 
\begin{equation}
\mathbf{H} vec(\mathbf{X}) = vec(\mathbf{E}),
\label{eq:linear}
\end{equation}
% where $\mathbf{H} = \mathbf{I}_\mathbf{A} diag(vec(\mathbf{B})) \mathbf{C}^\mathrm{T}_\mathbf{I} diag(vec( \mathbf{B})) \mathbf{I}_{\mathbf{A}^{\mathrm{T}}} + \mathbf{I}_\mathbf{D}$.
where $\mathbf{H} = \mathbf{I} \otimes \mathbf{D} + (\mathbf{I} \otimes \mathbf{A}) diag(vec(\mathbf{B})) (\mathbf{C}^\mathrm{T}\otimes\mathbf{I}) diag(vec( \mathbf{B})) (\mathbf{I} \otimes \mathbf{A})$.
\end{theorem}
\begin{proof}
We can rewrite Eq.~(\ref{eq:P1}) in the vector form as 
\begin{equation}
vec(\mathbf{A} (\mathbf{B} \ast (( \mathbf{B} \ast (\mathbf{A}^{\mathrm{T}} \mathbf{X})) \mathbf{C}))) + vec(\mathbf{D} \mathbf{X}) = vec(\mathbf{E}).
\label{eq:P1_vec}
\end{equation}
For simplicity, we use the notation $\mathbf{A}_\mathbf{B} = \mathbf{A} \otimes \mathbf{B}$. Notice that $\mathbf{A}_\mathbf{B} \neq \mathbf{B}_\mathbf{A}$. By using $vec(\mathbf{A} \mathbf{X} \mathbf{B}) = (\mathbf{B}^\mathrm{T}_\mathbf{A}) vec(\mathbf{X}) $ and $\mathbf{x} \ast \mathbf{y} = diag(\mathbf{x})\mathbf{y}$, we have:
\begin{align}
& vec(\mathbf{A} (\mathbf{B} \ast (( \mathbf{B} \ast (\mathbf{A}^{\mathrm{T}} \mathbf{X})) \mathbf{C}))) \nonumber \\
= ~& {\mathbf{I}}_\mathbf{A} vec(\mathbf{B} \ast (( \mathbf{B} \ast (\mathbf{A}^{\mathrm{T}} \mathbf{X})) \mathbf{C})) \nonumber \\
= ~& {\mathbf{I}}_\mathbf{A} ( vec(\mathbf{B}) \ast vec( (\mathbf{B} \ast (\mathbf{A}^{\mathrm{T}} \mathbf{X})) \mathbf{C} ) ) \nonumber \\
= ~& {\mathbf{I}}_\mathbf{A} diag( vec(\mathbf{B}) ) vec( (\mathbf{B} \ast (\mathbf{A}^{\mathrm{T}} \mathbf{X}))\mathbf{C}) \nonumber \\
= ~& {\mathbf{I}}_\mathbf{A} diag( vec(\mathbf{B}) ) \mathbf{C}^{\mathrm{T}}_{\mathbf{I}} vec(\mathbf{B} \ast (\mathbf{A}^{\mathrm{T}} \mathbf{X})) \nonumber \\
= ~& {\mathbf{I}}_\mathbf{A} diag(vec(\mathbf{B})) \mathbf{C}^{\mathrm{T}}_{\mathbf{I}} diag(vec(\mathbf{B})) {\mathbf{I}}_{\mathbf{A}^{\mathrm{T}}} vec(\mathbf{X}).
\label{eq:P1_vec_term1}
\end{align}
\vspace{-1pt}
Similarly, we have $vec(\mathbf{D}\mathbf{X}) = {\mathbf{I}}_\mathbf{D} vec(\mathbf{X})$. Then substitute this and Eq.~(\ref{eq:P1_vec_term1}) into Eq.~(\ref{eq:P1_vec}), we arrive at Theorem \ref{theo:LS}.
\end{proof}

Using Theorem \ref{theo:LS}, letting $\mathbf{A} = \mathbf{Z}^{(v)}$, $\mathbf{B}= \mathbf{\Pi}^{(-v)}$, $\mathbf{C} = {\mathbf{W}^{(V+1)}}^{\mathrm{T}}\mathbf{W}^{(V+1)}$, $\mathbf{D} = \gamma \mathbf{P}^{(v)}$, and $\mathbf{E} = \mathbf{Z}^{(v)}( \mathbf{\Pi}^{(-v)} \ast ( \mathbf{F} \mathbf{W}^{(V+1)} ) )$, we can rewrite Eq.~(\ref{eq:gred_Wv}) as a linear system of Eq.~(\ref{eq:linear}). Then, since $\mathbf{H}$ is invertible, we have the solution in the vector form as $vec(\mathbf{W}^{(v)})= \mathbf{H}^{-1}vec(\mathbf{E})$. 

However, the computation of $\mathbf{H}$ is usually time consuming. Alternatively,  we can solve Eq.~(\ref{eq:gred_Wv}) iteratively by the conjugate gradient (CG) method, which only needs to perform matrix multiplications of Eq.~(\ref{eq:gred_Wv}), an explicit representation of the matrix $\mathbf{H}$ is not needed.

% Alternatively,  we can solve Eq.~(\ref{eq:gred_Wv}) iteratively by the conjugate gradient (CG) method \cite{kelley1999iterative}, which only requires the computation of matrix products in Eq.~(\ref{eq:gred_Wv}), an explicit representation of the matrix $\mathbf{H}$ is not needed.

% This yields a linear system $\mathbf{A}\mathbf{x} = \mathbf{b}$ and can be solved analytically (see Appendix). Alternatively, we can solve it numerically by the conjugate gradient (CG) method, which only requires the computation of matrix products in Eq.~(\ref{eq:gred_Wv}), an explicit representation of the matrix $\mathbf{A}$ is not needed.

% to alleviate the computational burden, 

% Appendix A shows that Eq.~(eq:gred_Wv) can be expressed in the linear form $\mathbf{A}\mathbf{X} = \mathbf{B}$.

% To solve Eq.~(\ref{eq:gred_Wv}) for $\mathbf{W}^{(v)}$, we start by rewriting it in the linear form $\mathbf{A}\mathbf{X} = \mathbf{B}$.
% which is a linear system and can be written in the form. AH = b. 
% This form a system of linear equations (ie in the linear form Ax= b), because the values 
\renewcommand{\algorithmicrequire}{\textbf{Input:}}
\renewcommand{\algorithmicensure}{\textbf{Output:}}
\begin{algorithm}[t]
% \small
\caption{~Algorithm to solve the problem in Eq.~(\ref{eq:mmvc_obj})}
\label{alg:MMVC}
\small
\begin{algorithmic}[1] 
\REQUIRE Dataset $\{\mathbf{X}^{(1)}, \cdots, \mathbf{X}^{(V)}\}$, number of clusters $K$, number of factors $R$, and regularized parameter $\gamma$
\ENSURE $\{ \mathbf{W}^{(v)} \in \mathbb{R}^{(1+D_v) \times R} \}$, $\mathbf{W}^{(0)} \in \mathbb{R}^{K \times R}$, $\mathbf{F} \in \mathbb{R}^{N \times K}$
\STATE Let $t=0$. Initialize $\{\mathbf{W}^{(v)}_t\}$, $\mathbf{F}_t$ s.t. $\mathbf{F}^{\mathrm{T}}_t \mathbf{F}_t = \mathbf{I}$
\STATE Allocate $\mathbf{Z}^{(v)} = [\mathbf{1}; \mathbf{X}^{(v)}] \in \mathbb{R}^{(1+D_v) \times N}$
\WHILE{not converge}
\FOR{$v:=1~\text{to}~V$}
    \STATE Calculate the diagonal matrix $\mathbf{P}^{(v)}_{t+1}$ by $\mathbf{W}^{(v)}_t$  
    \STATE Update $\mathbf{W}^{(v)}_{t+1}$ by solving Eq.~(\ref{eq:gred_Wv})
\ENDFOR
\STATE Calculate the diagonal matrix $\mathbf{P}^{(V+1)}_{t+1}$ by $\mathbf{W}^{(V+1)}_t$
\STATE Update $\mathbf{W}^{(V+1)}_{t+1}$ by solving Eq.~(\ref{eq:opti_W0})
\STATE Calculate $\mathbf{F}_{t+1} = \mathbf{U}[\mathbf{I}; \mathbf{0} ] \mathbf{V}^\mathrm{T}$, where $\mathbf{U}$ and $\mathbf{V}$ are obtained by SVD on $\mathbf{\Pi}_t {\mathbf{W}_t^{(V+1)}}^{\mathrm{T}}$
\STATE t = t+1
\ENDWHILE
\end{algorithmic}
\end{algorithm}
% \vspace{-10pt}

% \begin{equation}
% \frac{ \partial \mathcal{J} }{ \partial \mathbf{W}^{(v)} } = \mathbf{Z}^{(v)} (\mathbf{\Pi}^{(-v)} \ast (( ( \mathbf{\Pi}^{(-v)} \ast (\mathbf{Z}^{(v)^\mathrm{T}} \mathbf{W}^{(v)})) {\mathbf{W}^{(0)}}^{\mathrm{T}} - \mathbf{F} ) \mathbf{W}^{(0)}))
% \end{equation}
\textbf{Update $\mathbf{W}^{(V+1)}$}: By keeping all other variables fixed and minimizing over $\mathbf{W}^{(V+1)}$, the optimization problem is 
\begin{equation}
\mathcal{J} = \min_{\mathbf{W}^{(V+1)}}~\| \mathbf{\Pi} {\mathbf{W}^{(V+1)}}^{\mathrm{T}} - \mathbf{F}\|^2_F + \gamma \|\mathbf{W}^{(V+1)}\|_{2,1}.
\label{eq:opti_W0_P2}
\end{equation}
Using techniques similar to above and setting the derivative of Eq.~(\ref{eq:opti_W0_P2}) w.r.t. $\mathbf{W}^{(V+1)}$ equal to zero, it yields:
% Likewise, we can solve the following problem to approximate the solution of Eq.~(\ref{eq:opti_W0_P2}):
% % \vspace{-15pt}
% \begin{equation}
% \mathcal{J} = \min_{\mathbf{W}^{(0)}}~\| \mathbf{\Pi} {\mathbf{W}^{(0)}}^{\mathrm{T}} - \mathbf{F}\|^2_F + \gamma tr ({\mathbf{W}^{(0)}}^\mathrm{T} \mathbf{P}^{(0)} \mathbf{W}^{(0)}),
% \label{eq:opti_W0_P2_approx}
% \end{equation}
% where $\mathbf{P}^{(0)}$ is a diagonal matrix with the $i$-th diagonal element as $p^{(0)}_{ii} = \frac{1}{2\|\mathbf{w}^{(0)i}\|_2}$.
% % \vspace{-5pt}
\begin{align}
\mathbf{W}^{(V+1)}\mathbf{\Pi}^\mathrm{T}\mathbf{\Pi} + \gamma_2 \mathbf{P}^{(V+1)} \mathbf{W}^{(V+1)} = \mathbf{F}^\mathrm{T}\mathbf{\Pi},
\label{eq:opti_W0}
\end{align}
where $\mathbf{P}^{(V+1)}$ is a diagonal matrix with the $i$-th diagonal element as $p^{(V+1)}_{ii} = \frac{1}{2\|\mathbf{w}^{(V+1)i}\|_2}$. The Eq.~(\ref{eq:opti_W0}) is the Sylvester equation that can be solved by several numerical approaches, here we use the lyap function in MATLAB.
% where $\mathbf{V}$ is a diagonal matrix with the $i$-th diagonal element as 
% \begin{equation}
% v_{ii} = \frac{1}{2 \left \|\mathbf{w}^{(0)}_i \right\|_2}.
% \end{equation}
% Note that the derivative of $\| \mathbf{W}^{(0)}\|_{2,1}$ is computed to $2\mathbf{Q}^{(v)}\mathbf{W}^{(0)}$ (see \emph{e.g.} \cite{nie2010efficient}).

\textbf{Update $\mathbf{F}$}: By keeping all other variables fixed and minimizing over $\mathbf{F}$, we need to solve the following problem:
\begin{equation}
\mathcal{J} = \min_{\mathbf{F}^{\mathrm{T}} \mathbf{F} = \mathbf{I}}~ \| \mathbf{\Pi} {\mathbf{W}^{(V+1)}}^{\mathrm{T}} - \mathbf{F}\|^2_F.
\label{eq:opti_F}    
\end{equation}
% In order to solve Eq.~(\ref{eq:opti_F}), we introduce the following theorem \cite{wang2013multi}, which provides the solution to Eq.~(\ref{eq:opti_F}) in a closed form.
The following theorem provides the analytical solution to this problem, the proof of which
can be found in \cite{wang2013multi}.
\begin{theorem}\label{theo:svd}
Given any matrix $\mathbf{A} \in \mathbb{R}^{N \times K} (K \leq N)$ and its singular value decomposition (SVD) $\mathbf{A} = \mathbf{U} \mathbf{\Lambda} \mathbf{V}^\mathrm{T}$, the solution of the optimization problem: $\min_{\mathbf{B}^{\mathrm{T}} \mathbf{B} = \mathbf{I}} \| \mathbf{A} - \mathbf{B}\|^2_F$ is given by $\mathbf{B} = \mathbf{U}[\mathbf{I}; \mathbf{0} ] \mathbf{V}^\mathrm{T}$, where $\mathbf{I}$ is the identity matrix with size $K$, $\mathbf{0} \in \mathbb{R}^{(N-K) \times K}$ is a matrix with all zeros.
\end{theorem}
% \begin{proof}
% The proof can be found in \cite{wang2013multi}.
% \end{proof}
Using Theorem \ref{theo:svd} and letting the SVD of $\mathbf{\Pi} {\mathbf{W}^{(V+1)}}^{\mathrm{T}}$ be $\mathbf{U} \mathbf{\Lambda} \mathbf{V}^\mathrm{T}$, we have the solution of Eq.~(\ref{eq:opti_F}) as $\mathbf{F} = \mathbf{U}[\mathbf{I}; \mathbf{0} ] \mathbf{V}^\mathrm{T}$.

\textbf{Convergence Analysis.}
We partition the Eq.~(\ref{eq:mmvc_obj}) into $V+2$ subproblems and each of them is a convex problem with respect to one variable. By solving the subproblems alternatively, it guarantees that we can find the optimal solution to each subproblem and finally, Algorithm \ref{alg:MMVC} can converge to a local minima of the objective function in Eq.~(\ref{eq:mmvc_obj}). The proof can be derived in a similar manner as in \cite{nie2010efficient}. \\
\textbf{Computational Analysis.}
In Algorithm \ref{alg:MMVC}, step 4 solves a SVD problem on a matrix of size $N \times K$. Because in typical clustering tasks, the number of clusters $K$ is usually small, step 4 can be computed efficiently by many off-the-shelf numerical packages. Step 6 and step 9 are computationally trivial. In step 7, instead of computing the matrix inverse with cubic complexity and explicitly forming the matrix, we can solve the problem by the CG method with superlinear complexity. Step 10 solves the Sylvester equation with cubic complexity in $R$ (in ideal case, usually it is very small value).

% \textbf{Clustering Rule.} Given the input multi-view data $\{ \mathbf{X}^{(v)}\}_{v=1}^V$, we can compute the cluster indication
% matrix $\mathbf{F}$ by Algorithm \ref{alg:MMVC}. Upon solution, we perform K-means clustering on $\mathbf{F}$ to get the clustering result.
\section{Experiments} \label{sec:exp}
In this section, we compare the proposed {\ours} approach with eight baseline methods over three real-world datasets.

\subsection{Datasets}
\begin{itemize}
    \item \textbf{FOX} and \textbf{CNN}\footnote{https://sites.google.com/site/qianmingjie/home/datasets/}: These two datasets were crawled from FOX and CNN web news. %\cite{qian2014unsupervised} 
    $1,523$ articles from FOX news are categorized into 4 classes, and $2,107$ articles from CNN news are categorized into 7 classes.
    Each article are represented by two views, the text view and image view.
	We filter out words with frequency less than $1\%$  which results in $2,711$ features for FOX and $1,686$ features for CNN.
	For image features, seven groups of color features and five textural features are used, which results in $ 996 $ features for both datasets.
	\item \textbf{3Sources}\footnote{http://mlg.ucd.ie/datasets/3sources.html}: The news data from three sources, BBC, Reuters, and The Guardian, consist of 948 news articles covering 416 distinct news stories in six topics. Of these stories, 169 were reported in all three sources.
    We use these 169 shared stories in our experiments.

\end{itemize}
\subsection{Comparison methods}
\begin{itemize}
 \item \textbf{SEC} is a single-view spectral embedding clustering method \cite{SEC}. It uses both local and global discriminative information and embeds the cluster assignment matrix in a linear space spanned by the high-dimensional data. 
\item \textbf{CoRegSc} is the centroid based co-regularized multi-view spectral clustering method proposed by \cite{coregSC}.
\item \textbf{AMGL} is the most recent multi-view spectral clustering method \cite{AMGL}, which learns the weight for each graph automatically without introducing additional parameters.
\item \textbf{RMKMC} is the robust multi-view K-means clustering method \cite{RMKMC}, which aims to integrate heterogeneous representations of large-scale data.
\item \textbf{MultiNMF} is the multi-view clustering method based on joint NMF \cite{multiNMF}. It formulates a joint NMF process with the constraint that pushes clustering solution of each view toward a common consensus.
\item \textbf{KCTD} is the kernel concatenation tensor decomposition approach that first construct the tensor by stacking the kernel matrices of all the views. The partial symmetric CP decomposition is then applied to extract the latent feature for the clustering \cite{shao2015clustering}.
\item \textbf{SCMV-3DT} is the one of the most recent third-order  tensor based multi-view clustering method proposed by \cite{SCMV_3DT}. 
    By using t-product based on the circular convolution, the multi-view tensor data is reconstructed by itself with sparse and low-rank penalty.
\item \textbf{{\ours}} is our proposed multi-linear multi-view clustering method. To study the effects of feature selection contributed by the sparsity-inducing norm in the model, we replace the $\ell_{2,1}$-norm by the Forbenius norm and denote this model by \textbf{{\ours}-F}.
\end{itemize}

\subsection{Experimental Setup}
In the experiments, we employ two widely used metrics to measure the clustering performance: accuracy (ACC) and normalized mutual information (NMI). Since SEC is designed for single-view data, we first concatenate all the views and then apply SEC on the concatenated views. For methods involve inter-view feature interactions, we normalize features in each view such that the sum of the squared feature values equals to one. There are two parameters in the proposed {\ours}, $R$ and $\gamma$. We empirically set $\gamma$ to $0.01$ and do a linear search for $R$ in $[10, 20, \dots, 50]$. We tune the parameters of each baseline methods using the strategy mentioned in the corresponding paper. 
For all the spectral clustering based algorithms, we construct RBF kernel matrices with kernel width $\theta$ to be the median distance among all the instances. We construct each graph by selecting 10-nearest neighbors among raw data. Each experiment was repeated for 20 times, and the average performance of each  metric  in  each  data  set  were  reported. All  experiments are conducted on machines with Intel Xeon 6-CoreCPUs of 2.4 GHz and 64 GB RAM.

\begin{table}[t]
\centering
\caption{Clustering results on three datasets (\%).}
\label{tab:exp_results}
\begin{adjustbox}{max width=2\columnwidth}
\begin{tabular}{|c|c|c|c|c|c|c|}
\hline
             & \multicolumn{2}{c|}{FOX}                            & \multicolumn{2}{c|}{CNN}                            & \multicolumn{2}{c|}{3Sources}                        \\ \hline
             & ACC             & NMI     & ACC             & NMI             & ACC             & NMI           \\ \hline
SEC          & 43.73          & 8.87        & 25.15          & 4.62              & 47.13          & 20.43          \\ \hline
CoregSC      & 54.74          & 24.37         & 27.66          & 7.73              & 56.05          & 52.44                \\ \hline
AMGL         & 42.03          & 3.52         & 23.46          & 2.99          & 37.57          & 11.13            \\ \hline
RMKMC        & 61.92          & 31.51        & 21.16          & 2.47               & 44.47          & 26.04          \\ \hline
MultiNMF     & 76.93          & 67.10         & 40.14          & 29.67             & 50.78          & 46.21            \\ \hline
KCTD         & 46.24          & 20.79          & 27.56          & 8.95         & 41.93          & 34.68        \\ \hline
SCMV-3DT     & 74.73          & 58.90         & 32.32          & 26.23         & 56.70          & 47.36             \\ \hline
MMC-F        & 71.15          & 52.77         & 56.60          & 43.92           & 51.86          & 31.14      \\ \hline
MMC         & \textbf{79.62} & \textbf{71.39} & \textbf{57.88} & \textbf{47.30} & \textbf{60.58} & \textbf{52.83} \\ 
\hline
\end{tabular}
\end{adjustbox}
\vspace{-5pt}
\end{table}

\subsection{Results}
Table~\ref{tab:exp_results} shows the clustering results for all the methods. By comparison, we have the following observations. First, the proposed {\ours} method outperforms all the other methods on all the datasets. This is mainly because {\ours} can learn the latent subspaces embedded within the full-order interactions between multiple views, while other methods fail to explore the explicit correlations among multi-view features. Moreover, it can be found that {\ours} significantly outperforms other methods on the FOX and CNN datasets. The reason behind is that both datasets consist of completely different types of features (image and text). By exploiting the inter-view feature interactions, we are able to boost the clustering performance. Further, it can be found that {\ours} always performs better than {\ours}-F, which empirically shows the effectiveness of feature selection in high-dimensional data.

\begin{figure}[t]
    \centering
    \includegraphics[width=0.6\linewidth]{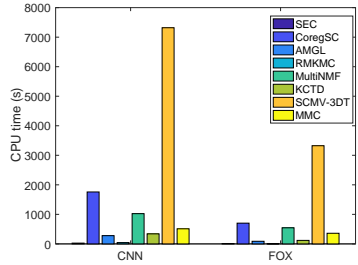}
    \caption{CPU time on large-scale datasets.}
    \label{fig:cputime}
\end{figure}

\begin{figure}[]
    \centering
    \includegraphics[width=0.7\linewidth]{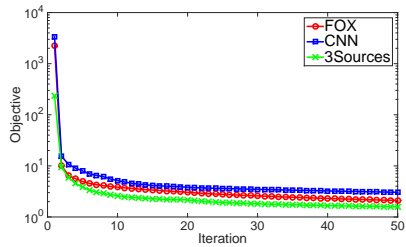}
    \caption{Objective \textit{v.s.} the number of iterations.}
    \label{fig:obj}
    \vspace{-5pt}
\end{figure}

The CPU runtime of comparison methods on large-scale datasets are reported in Figure~\ref{fig:cputime}. It is observed that {\ours} is much faster than CoregSC, MultiNMF and SCMV-3DT. The three methods that are exactly the top-$3$ baseline methods that can almost always perform better than the other baseline methods. Furthermore, the convergence speed of {\ours} on three datasets is given in Figure~\ref{fig:obj}. It shows that {\ours} algorithm can converge within few iterations.

% Besides, Figure~\ref{fig:obj} provides the convergence speed of {\ours} on the three datasets. It shows that {\ours} algorithm can converge within few iterations.

% It is observed that {\ours} is much faster than CoregSC, MultiNMF and SCMV-3DT for all the datasets except 3Sources. 
% The three methods that are exactly the top-$3$ baseline methods that can almost always perform better than the other baseline methods. 
% On the 3Sources dataset, out method takes a little bit more time than other methods, due to the number of instances is much fewer than the number of features.
\section{Conclusion} \label{sec:conclude}
In this paper, we introduced a novel multi-view clustering algorithm {\ours} based on tensor analysis, which can efficiently learn the underlying latent structure embedded in the full-order feature interactions among multiple views, without the need to construct the higher-order tensor data physically. Experiments on three real-world datasets demonstrate that {\ours} outperforms several state-of-the-art methods.

\section*{Acknowledgment}
% \small
This work is partially supported by NSFC (No. 61503253 and 61672313), NSF (No. IIS-1716432, IIS-1750326, IIS-1526499, IIS-1763325, IIS-1763365, and CNS-1626432), NSF-Guangdong (No. 2017A030313339), FSU through the startup package and FYAP award, and ONR N00014-18-1-2585, 1R01LM012607, 1R01AI130460, P50MH113840, R01AI116794, 7R01LM009012.

% \balance
\bibliographystyle{IEEEtran}
\bibliography{reference}
\end{document}